\newif\ifabstract
\abstracttrue
 \abstractfalse 
\newif\iffull
\ifabstract \fullfalse \else \fulltrue \fi

\documentclass[11pt,onecolumn]{article}
\usepackage[utf8]{inputenc} 
\usepackage[T1]{fontenc}    
\usepackage{hyperref}       
\usepackage{url}            
\usepackage{booktabs}       
\usepackage{amsfonts}       
\usepackage{nicefrac}       
\usepackage{microtype}      
\usepackage{bm}
\usepackage{algorithm,algpseudocode}
\usepackage{algorithmicx}
\usepackage{amsmath}
\usepackage{theorem}
\usepackage{graphicx,epsfig,amsfonts,bbm,epstopdf,tabularx}
\usepackage{xspace}
\usepackage{amssymb}
\usepackage{color}
\usepackage{xspace}
\usepackage{setspace}
\usepackage{amsmath}
\usepackage{amsfonts}
\usepackage{amssymb}
\usepackage{amstext}
\usepackage{amsmath}
\usepackage{xspace}
\usepackage{theorem}
\usepackage{graphicx}
\usepackage{url}
\usepackage{graphics}
\usepackage{colordvi}
\usepackage{colordvi}
\usepackage{subfigure}

\usepackage{setspace}
\usepackage{color}
\usepackage{cite}
\usepackage[disable]{todonotes}
\usepackage{hhline}
\usepackage{array}
\usepackage{pifont}
\usepackage{enumerate}

\usepackage{graphicx,epsfig,amsfonts,bbm,epstopdf,tabularx}

\usepackage{bm}

\usepackage{url}

\usepackage{xcolor}
\usepackage[comma,numbers,square,sort&compress]{natbib}
\usepackage{algorithm,algpseudocode}
\usepackage{algorithmicx}
\usepackage{stmaryrd}
\usepackage{multirow}
\usepackage{graphicx}
\usepackage{arydshln}
\usepackage{mathtools}
\usepackage{tabulary}
\usepackage{booktabs}
\usepackage{subfigure}
\usepackage{xspace}

\advance \oddsidemargin by -0.8in
\newcommand{\myparskip}{3pt}
\parskip \myparskip

        {\hspace*{\fill}$\Box$\par\vspace{4mm}}

\newcommand{\be}{\begin{equation}}
\newcommand{\ee}{\end{equation}}
\newcommand{\bse}{\begin{subequations}}
\newcommand{\ese}{\end{subequations}}
\def\bee#1\eee{\begin{align}#1\end{align}}
\newcommand{\nnb}{\nonumber}

\newtheorem{theorem}{Theorem}[section]
\newtheorem{lemma}[theorem]{Lemma}

\newenvironment{proof}{\par \smallskip{\bf Proof:}}{\hfill\stopproof}
\def\stopproof{\square}
\def\square{\vbox{\hrule height.2pt\hbox{\vrule width.2pt height5pt \kern5pt
\vrule width.2pt} \hrule height.2pt}}

\renewcommand{\phi}{\varphi}

\mathchardef\hyphen="2D

\ifodd 0
\newcommand{\rev}[1]{{\color{red}#1}}
\newcommand{\com}[1]{\textbf{\color{blue} (COMMENT: #1)}}
\else
\newcommand{\rev}[1]{#1}
\newcommand{\com}[1]{}
\fi

\newcommand{\rew}{\textsf{REW}\xspace}
\newcommand{\ew}{\textsf{EW}\xspace}

\allowdisplaybreaks[4]

\begin{document}

\title{Recursive Exponential Weighting for Online Non-convex Optimization}
\author{Lin Yang\thanks{Information Engineering Department, The Chinese University of Hong Kong. Email: {\tt yl015@ie.cuhk.edu.hk}.} \and Cheng Tan \thanks{Information Engineering Department, The Chinese University of Hong Kong. Email: {\tt tancheng1987love@163.com}.} \and Wing Shing Wong \thanks{Information Engineering Department, The Chinese University of Hong Kong. Email: {\tt wswong@ie.cuhk.edu.hk}.}}

\maketitle

\begin{abstract}
  In this paper, we investigate the \emph{online non-convex optimization} problem which generalizes the classic \emph{online convex optimization} problem by relaxing the convexity assumption on the cost function.
  For this type of problem, the classic \emph{exponential weighting} online algorithm has recently been shown to attain a sub-linear regret of $O(\sqrt{T\log T})$.
  In this paper, we introduce a novel recursive structure to the online algorithm to define a \emph{recursive exponential weighting algorithm}
 that attains a regret of $O(\sqrt{T})$, matching the well-known regret lower bound.
  To the best of our knowledge, this is the first online algorithm with provable $O(\sqrt{T})$ regret
  for the online non-convex optimization problem.
\end{abstract}

\section{Introduction} \label{sec:introduction}
The Online Convex Optimization (OCO) framework has  widely influenced  the online learning community
since the seminal work by Zinkevich \cite{Zinkevich03}.
OCO is modeled as a repeated game composed of $T$ iterations.  At iteration $t$, the player chooses a point $\boldsymbol{x}_{t}$ from a bounded convex decision set
$\mathcal{K}\subset \mathbb{R}^{n}$; after the choice is committed, a bounded convex cost function $f_{t}: \mathcal{K}\mapsto \mathbb{R}$ is revealed to the player.
The goal of the player is to minimize the \textbf{regret}, which is defined to be the difference
between the online cumulative cost and the cumulative cost using an optimal offline choice in hindsight.
This model can be applied to many real-world problems, such as online routing \cite{Awerbuch08}, ad selection for search engines \cite{Wauthier13}
and spam email filtering \cite{Guzella09,Sculley07}, \textit{etc}.
It is well known that the tight lower bound of the regret for the OCO problem is $O(\sqrt{T})$ \cite{Hazan16} and
researchers have proposed many different online algorithms whose regret attains this lower bound,
including the Online Gradient Decent (OGD) method \cite{Zinkevich03}, the Stochastic Gradient Decent (SGD) method \cite{robbins1951,rakhlin2012making,shamir2013stochastic,hazan2014beyond}, the online Newton step method and many regularization-related methods \cite{kivinen1998,grove2001} (see the survey paper \cite{Hazan16}).
If one further assumes that the cost function $f_t$ is strictly convex, the regret can even be reduced to $O(\log T)$ \cite{Hazan2007}.

For the OCO problem, one of the most natural extensions is to relax the convexity assumption on the cost function, i.e., $f_t$ is allowed to be non-convex. This extension brings out the online non-convex optimization problem, which is necessitated by some important applications.
For example, in the portfolio selection problem \cite{Cover91,kalai2002}, the decision maker (e.g., the trader)
chooses a distribution of her wealth allocation over $n$ assets $\bm{x}_{t}$ at each round.
At the ending of every round, the adversary chooses the market returns for the assets with positive values.
Due to non-convex diversification constraints and non-convex transaction costs, the online portfolio selection problem would be non-convex \cite{uryasev2000conditional,krokhmal2002portfolio,quaranta2008robust,ardia2010differential},
and thus the traditional OCO framework fails in modeling such case. For more examples of non-convex applications, one can refer to \cite{Ertekin2011} \cite{Gasso2011} which discuss non-convex online Support Vector Machines (SVMs) \cite{zhang2005gene} and non-convex Neyman-Pearson classification, respectively.

Online non-convex optimization is not a new problem and there are plenty prior works on it reported in the literature.  Among them,
\cite{Ertekin2011} and \cite{Gasso2011} provided respective heuristic online algorithms, but neither of them are rigorously shown to satisfy any regret bound.
In \cite{Hazan2012}, Hazan and Kale considered a special online non-convex optimization problem
where the cost function is assumed to be submodular.  For such a cost function, their proposed online algorithm can attain the regret of $O(\sqrt{T})$.
In \cite{Zhang15}, the authors investigated an online bandit
learning problem with non-convex losses. The cost function is again a special non-convex function,
defined as the composition of a non-increasing scalar function with a linear function of small variation.
They developed an online algorithm of $\tilde{O}(\textsf{poly}(d)T^{2/3})$ regret, where $\textsf{poly}(d)$ stands for a polynomial that takes the dimension of the decision set, $d$,
as argument.
The works that are most related to ours are those by Krichene \cite{Krichene15} and Maillard \cite{Maillard2010}.
Both of them applied the exponential weighting method to attain a regret of $O(\sqrt{T\log T})$.
To the best of our knowledge, no online algorithm with regret that achieves the well-known lower bound for the online non-convex
optimization problem has been reported until now.

%
%

This paper fills in this blank by proposing a novel online algorithm, called Recursive Exponential Weighting (\rew) and proving that it can attain the tight lower regret bound.
The idea of \rew is to divide the decision set 
into multiple subsets according to a \emph{layered} structure.
Any subset in the upper layer is divided into smaller subsets in the lower layer.
\rew recursively selects the subset from the  top layer to the  bottom layer until a decision point is identified.
In each layer, \rew uses the traditional Exponential Weighting (\ew) method \cite{Auer02} to select a subset in the lower payer. By properly partitioning subsets and setting the
subset-selecting probabilities, we prove that our new proposed \rew online algorithm can asymptotically attain a regret of
$O(\sqrt{T})$, which is the lower bound of the regret. Therefore, \rew is asymptotically optimal for the general online
non-convex optimization problem.

\section{Problem Setting of Online Non-Convex Optimization}
Similar to the OCO framework, our online non-convex optimization problem can be seen as a structured repeated game.
At each iteration $t$, the player is required to choose a decision $\boldsymbol{x}_{t}$
from a continuous and bounded decision set $\mathcal{K}\subseteq \mathbb{R}^n$.
After the player commits to a decision point at slot $t$, the adversary chooses
a cost function $f_{t}(\boldsymbol{x})$ from $\mathcal{F}$. $\mathcal{F}$ is a bounded family of cost functions $f_{t}(\boldsymbol{x}):~\mathcal{K}\mapsto [0,B]$, which are assumed to be non-negative and
Lipschitz continuous with parameter $L>0$, i.e.,
\begin{equation}
|f_{t}(\boldsymbol{x})-f_{t}(\boldsymbol{y})|\leq L\|\boldsymbol{x}-\boldsymbol{y}\|_2, \quad \forall \boldsymbol{x},\boldsymbol{y}\in
\mathcal{K}.  \label{equ:cost-function-Lipschitz}
\end{equation}
Note that in the OCO setting, the cost function is required to be convex, while the model in our paper takes into account
a more general class of cost functions including both convex and non-convex cost functions.

The whole cost function $f_{t}(\boldsymbol{x})$ is revealed
to the player only after a choice is made at time slot $t$.

At each time slot, the player needs to make decisions in an online fashion without knowing the current and future cost functions.
The common performance metric to evaluate any online algorithm is
the pseudo-regret\footnote{We call it regret in short in the rest of this paper.}, defined as
\begin{equation}
\textsf{regret}\xspace_{T} \overset{\text{def}}{=} \sup_{f_{1},\ldots,f_{T}\in \mathcal{F}}\left\{\sum_{t=1}^{T}f_{t}(\bm{x}_{t})-\min_{x\in \mathcal{K}}\sum_{t=1}^{T}f_{t}(\bm{x})\right\},
\end{equation}
which is the cumulative difference between the cost of the online algorithm and the cost of the best fixed offline decision.

Our goal is to design online algorithms to our online non-convex problem
and try to minimize the regret. Later on in Sec.~\ref{sec:rew},
we design an online algorithm, called Recursive Exponential Weighting (\rew), and we show that
its regret is $O(\sqrt{T})$ asymptotically, which attains the lower bound of the regret \cite{Hazan16}.



\section{Recursive Exponential Weighting Online Algorithm} \label{sec:rew}
In this section, we propose a novel weighting method, which is called the Recursive Exponential Weighting (\rew).
Intuitively, \rew is based on the conceptual idea of grouping highly correlated decisions into one set and
adopt a \textit{divide-and-conquer} method.
Before we solve the general online non-convex problem,
we first introduce how to discretize the decision set $\mathcal{K}$ in Sec.~\ref{subsec:set-discretization}.

\subsection{Set Discretization} \label{subsec:set-discretization}
In this subsection, we introduce a very straightforward method to discretize the decision set $\mathcal{K}$.

Because the decision set $\mathcal{K}$ is bounded, we can find a bounded cube of length $D$, denoted by $\mathcal{D}$, that can cover $\mathcal{K}$ entirely. As shown in Figure \ref{fig:non-regular-decision-set}, we partition $\mathcal{D}$ into smaller equal-size sub-cubes with edge length being $\frac{D}{2^{m}}$.
$m$ specifies the granularity of set discretization.
Assume the decision set is $n$-dimensional. The total number of sub-cubes is equal to $2^{mn}$. For simplicity, each sub-cube is indexed by a distinct $n$-dimensional vector $\boldsymbol{i}=(i_{1},i_{2},\cdots,i_{n})$, where $1\leq i_{j}\leq 2^{m}, j=1,2,\cdots,n$. In this way, a sub-cube can be denoted by $\mathcal{D}_{\boldsymbol{i}}$.
We denote the index set for sub-cubes which have overlap with the decision set $\mathcal{K}$ by $\mathcal{I}$, \textit{i.e.},
\begin{equation*}
\mathcal{I}\overset{\text{def}}{=}\{\boldsymbol{i}:~\mathcal{D}_{\boldsymbol{i}}\cap\mathcal{K}\neq \emptyset \}.
\end{equation*}

Fig.~\ref{fig:correlated-bandit-problem-setting} illustrates an example of index set $\mathcal{I}$ with $n=2$ and $m=3$, whose elements correspond to the overlapped sub-cubes in Fig. \ref{fig:non-regular-decision-set}.

In each overlapped sub-cube $\mathcal{D}_{\boldsymbol{i}}$, we randomly choose an overlapped point as the ``representative'' of $\mathcal{D}_{\boldsymbol{i}}$. At time slot $t$, once a sub-cube $\mathcal{D}_{\boldsymbol{i}}$ is chosen, the representative point associated with $\mathcal{D}_{\boldsymbol{i}}$ will be chosen as the final decision of the online algorithm. Correspondingly, the cost on the representative point, denoted by $c_t(\bm{i})$, will be incurred.

Under the above discretization method, the choices of the online algorithm will be reduced to the finite discrete set $\mathcal{I}$, and correspondingly, the original problem reduces to the classic expert problem with $|\mathcal{I}|$ experts. Certainly, the optimal point in hindsight may not lie among the representative points, so the above discretization method may result in an extra regret loss to the online algorithm. When we partition the decision space into very small sub-cubes, the cumulative cost of the optimal choice among representative points will approximate the optimal decision point over $\mathcal{K}$.

Because the cost function is Lipschitz continuous with parameter $L$, we have that
\begin{equation*}
|c_t(\bm{p}) - c_t(\bm{q})| \le L_{\mathsf{d}} ||\bm{p}-\bm{q}||_1,~~\bm{p},\bm{q}\in \mathcal{I},
\label{equ:discrete-Lipschitz-condition}
\end{equation*}
where $L_{\mathsf{d}}=2\sqrt{n}DL/(2^m)$ and $||\bm{p}-\bm{q}||_1 \overset{\text{def}}{=} \sum_{j=1}^{n} |p_j-q_j|$ is the one norm of vector $\bm{p}-\bm{q}$..

\begin{figure*}
\hspace{-0.5cm}
  \begin{minipage}[t]{0.01\linewidth}~~~
  \end{minipage}
  \begin{minipage}[t]{0.31\linewidth}
    \centering
    \includegraphics[width = \linewidth]{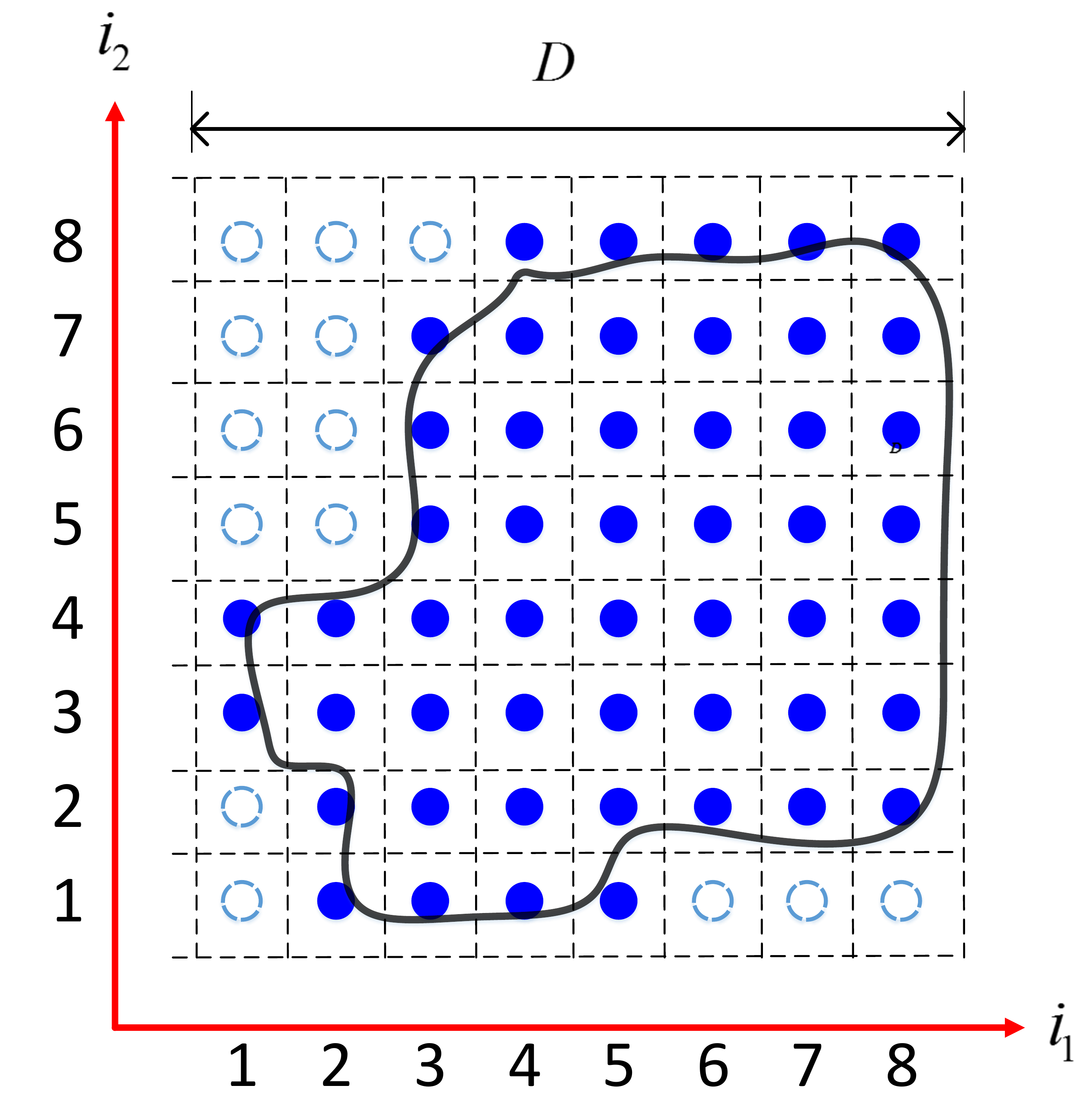}\\
  \caption{Set discretization for a general continuous decision set.}
  \label{fig:non-regular-decision-set}
  \end{minipage}
  \begin{minipage}[t]{0.01\linewidth}~~~
  \end{minipage}
    \begin{minipage}[t]{0.01\linewidth}~~~
  \end{minipage}
  \begin{minipage}[t]{0.31\linewidth}
    \centering
    \includegraphics[width = \linewidth]{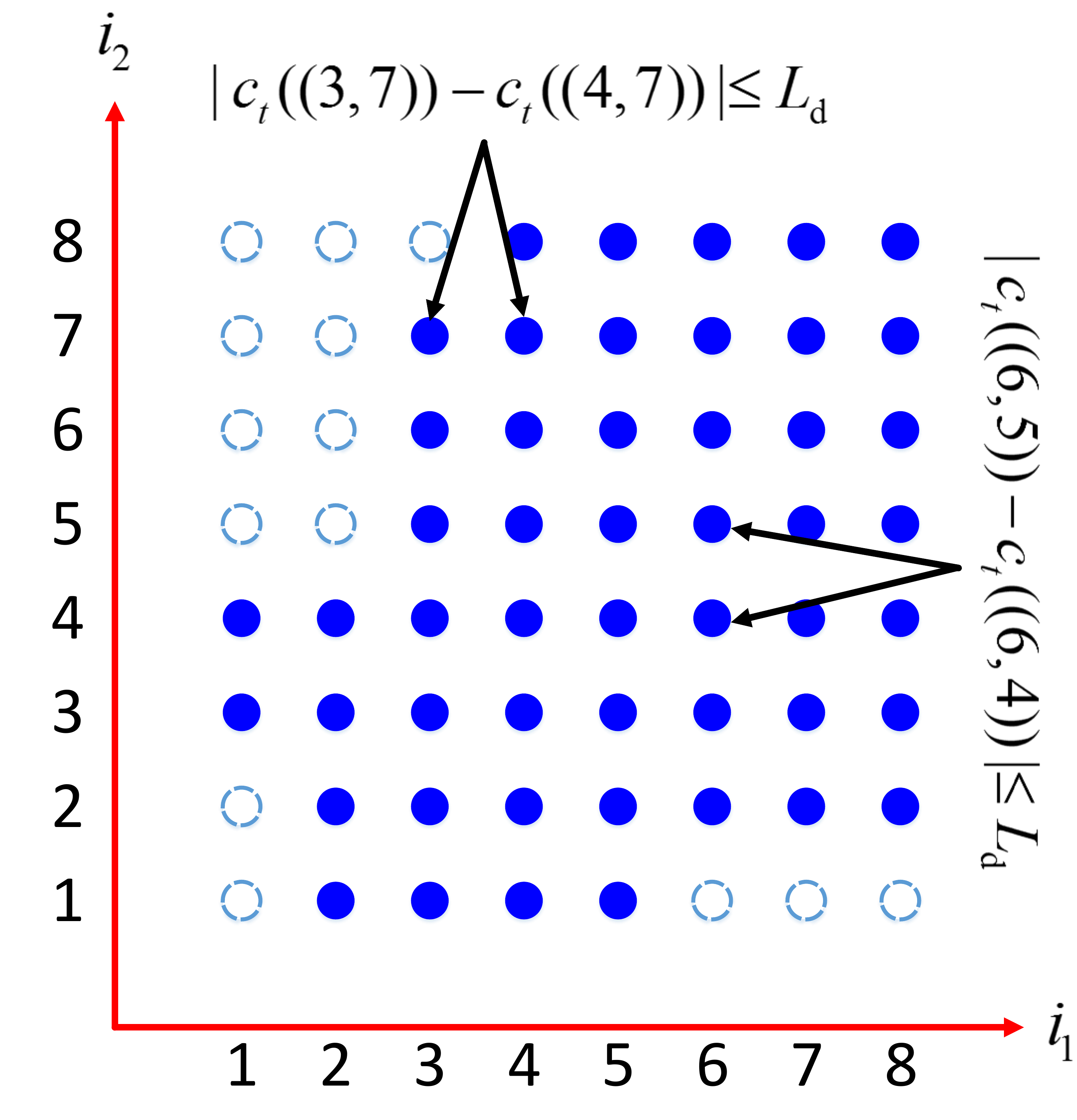}\\
   \caption{The index set for the overlapped sub-cubes when $n=2$ and $m=3$.}
  \label{fig:correlated-bandit-problem-setting}
  \end{minipage}
  \begin{minipage}[t]{0.01\linewidth}~~~
  \end{minipage}
  \begin{minipage}[t]{0.31\linewidth}
    \centering
    \includegraphics[width = \linewidth]{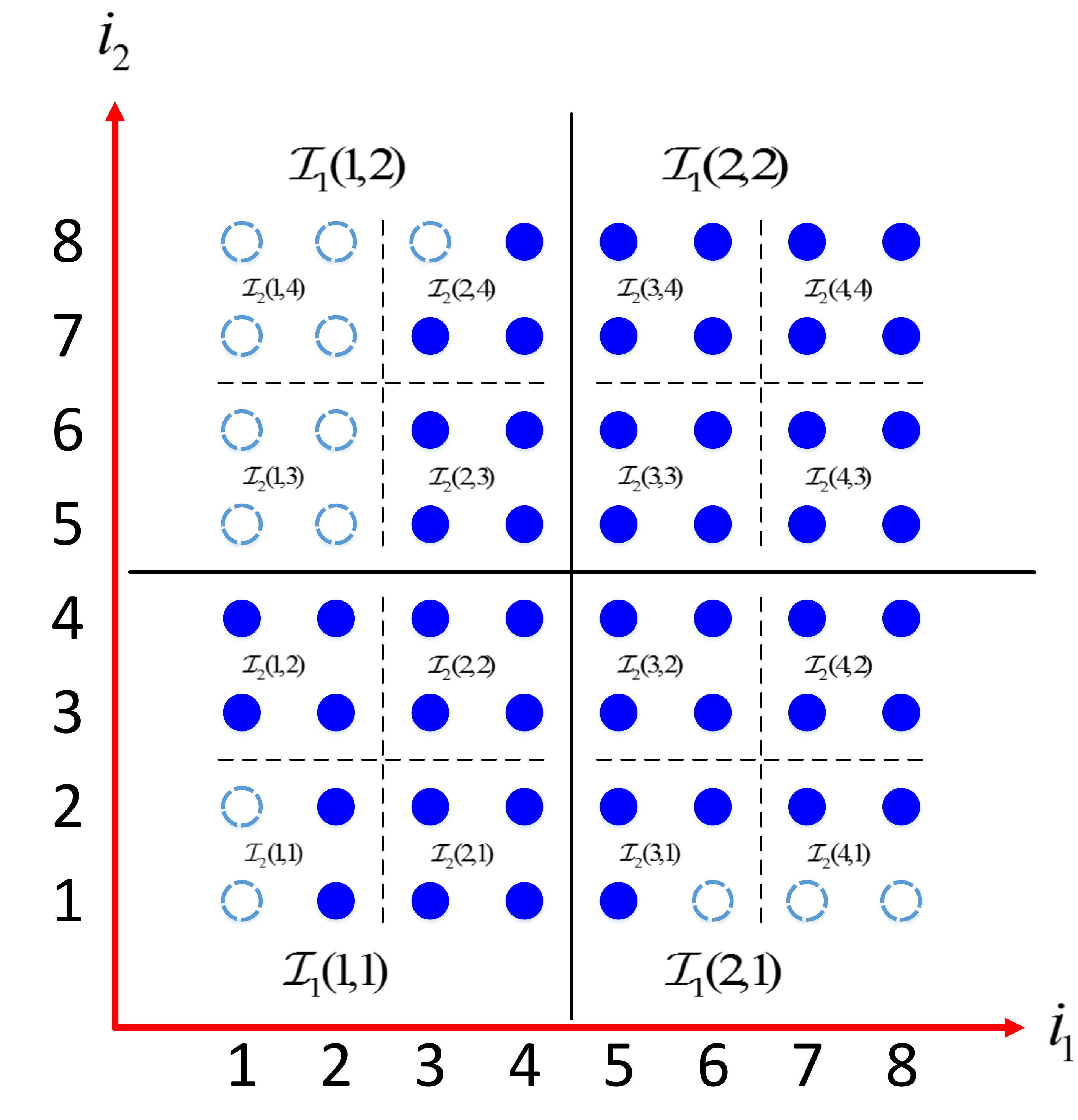}\\
  \caption{Set partition for the correlated bandit problem when $n=2$ and $m=3$.}
  \label{fig:correlated-bandit-problem-set-parition}
  \end{minipage}
  \begin{minipage}[t]{0.01\linewidth}~~~
  \end{minipage}
\end{figure*}

\subsection{Set Partition for the index set $\mathcal{I}$} \label{subsec:set-partition}
We do a set partition for the index set $\mathcal{I}$ according to different layers.
In layer $l \in [1,2,\cdots, m]$, a subset $\mathcal{I}_{l}(\boldsymbol{i})$ is defined to contain a group of neighbouring points, specifically,
\begin{equation*}
\begin{split}
&\mathcal{I}_{l}(\boldsymbol{i}) \overset{\text{def}}{=} \{\boldsymbol{p}\in \mathcal{I}:1+ (i_{j}-1)2^{m-l} \leq p_{j}\leq 2^{m-l} + (i_{j}-1)2^{m-l},j=1,2,\cdots,n\}.
\end{split}
\end{equation*}
where $\boldsymbol{i}\in \{(i_{1},i_{2},\ldots,i_{n}):1 \le i_1, i_2, \cdots, i_{n} \le 2^{l}\}$ denote the index of a subset in layer $l$. Then, the total $|\mathcal{I}|$ points are divided into $2^{nl}$ $l$-layer subsets whose size is at most $2^{(m-l)n}$.
For example, when $n=2$ and $m=3$, arms can be partitioned into $4$ subsets in layer $l=1$ or $16$ subsets in layer $l=2$,
as depicted in Fig.~\ref{fig:correlated-bandit-problem-set-parition}. Note that some subsets may be empty.

By convention, we regard the whole index set $\mathcal{I}$ as the only layer-0 subset, denoted by
$\mathcal{I}_{0}(\boldsymbol{1})=\mathcal{I}$. We should also note that the layer-$m$ subset contains at most one point, i.e.,
$|\mathcal{I}_{m}(\boldsymbol{i})|\leq 1$.

For simplicity, we use notation $\mathcal{U}_{l}(\boldsymbol{i})$ to denote the layer-$l$ subset containing a non-empty lower-layer subset $\mathcal{I}_{l+1}(\boldsymbol{i})$, \textit{i.e.},
\rev{
\begin{equation*}
\mathcal{U}_{l}(\boldsymbol{i})\overset{\text{def}}{=} \left\{\boldsymbol{j}: \mathcal{I}_{l+1}(\boldsymbol{j})\subset \mathcal{I}_{l}(\boldsymbol{i}),~\mathcal{I}_{l+1}(\boldsymbol{j})\neq \emptyset\right\}.
\end{equation*}
}
We use $\mathcal{M}_{l}(\boldsymbol{k})$ to denote the layer-$l$ subset containing point $\boldsymbol{k} \in \mathcal{I}$. We use $\mathcal{D}_{l}(\boldsymbol{i})$ to denote the index set of non-empty layer-$(l+1)$ subsets within $\mathcal{I}_{l}(\boldsymbol{i})$, \textit{i.e.},
\begin{equation}
\mathcal{D}_{l}(\boldsymbol{i})\overset{\text{def}}{=} \left\{\boldsymbol{j}: \mathcal{I}_{l+1}(\boldsymbol{j})\subset \mathcal{I}_{l}(\boldsymbol{i}),~\mathcal{I}_{l+1}(\boldsymbol{j})\neq \emptyset\right\}.
\end{equation}

\subsection{\rev{The Recursive Exponential Weighting (\rew) Online Algorithm}} \label{subsec:REW-algorithm}
In the online learning field, the expert problem \cite{cesa1997use} is a classical problem.
A general idea to attain a sublinear regret for the expert problem is to give more preference
to the expert with smaller cumulative cost in a stochastic manner \cite{Littlestone1994},
which can be implemented by the idea of Exponential Weighting.
In the subsection, we propose our novel Recursive Exponential Weighting online algorithm.
In each iteration, \rew chooses an point from $\mathcal{I}$ by recursively choosing a non-empty subset from the top layer ($l=0$)
to the bottom layer ($l=m$). In each layer, \rew uses the idea of exponential weighting to determine the probability of selecting a subset.
More specifically, if a subset $\mathcal{I}_l(\boldsymbol{s})$ in the $l$-th layer has been chosen,
then \rew chooses a non-empty subset on the $(l+1)$-th layer within $\mathcal{I}_l(\boldsymbol{s})$.
Unlike the Hedge algorithm whose choosing probability is based on the cumulative cost,
the choosing probability in \rew is based on the \emph{cumulative expected normalized cost}.
After revealing the cost of all points in iteration $t$,
the \emph{expected normalized cost} of subset $\mathcal{I}_{l+1}(\boldsymbol{i})$ (assume $\mathcal{I}_{l+1}(\boldsymbol{i})\neq \emptyset$)
at iteration $t$ is defined as
\begin{equation*}
\begin{split}
&\bar{c}_{l+1,t}(\boldsymbol{i}) \overset{\text{def}}{=} \mathbb{E}\left[\frac{c_{t}(\boldsymbol{I}_{t})-\min_{\boldsymbol{k}\in \mathcal{U}_{l}(\boldsymbol{i})}c_{t}(\boldsymbol{k})}{n2^{m-l}L_{\mathsf{d}}}|\boldsymbol{I}_{t}\in \mathcal{I}_{l+1}(\boldsymbol{i})\right]\\
=&\sum_{\boldsymbol{k}\in\mathcal{I}_{l+1}(\boldsymbol{i})}\!\!\frac{c_{t}(\boldsymbol{k})\!-\!\min_{\boldsymbol{k}\in \mathcal{U}_{l}(\boldsymbol{i})}c_{t}(\boldsymbol{k})}{n2^{m-l}L_{\mathsf{d}}}\cdot\Pr\left[\boldsymbol{I}_{t}\!=\!\boldsymbol{k}|\boldsymbol{I}_{t}\in \mathcal{I}_{l+1}(\boldsymbol{i})\right],
\end{split}
\end{equation*}
where $\Pr\left[\boldsymbol{I}_{t}=\boldsymbol{k}|\boldsymbol{I}_{t}\in \mathcal{I}_{l+1}(\boldsymbol{i})\right]$ is the
probability of selecting an index point $\bm{k}\in \mathcal{I}_{l+1}(\boldsymbol{i})$ conditioning on that $\mathcal{I}_{l+1}(\boldsymbol{i})$ is selected at slot $t$,
and can be calculated as
\begin{equation*}
\label{eq:pro_arm}
\begin{split}
&\Pr\left[\boldsymbol{I}_{t}=\boldsymbol{k}|\boldsymbol{I}_{t}\in \mathcal{I}_{l+1}(\boldsymbol{i})\right] \\
=& \prod_{i=l+2}^{m}\Pr\left[\boldsymbol{I}_{t}\in \mathcal{M}_{i}(\boldsymbol{k})|\boldsymbol{I}_{t}\in \mathcal{M}_{i-1}(\boldsymbol{k})\right].
\end{split}
\end{equation*}
It is easy to see that $0\leq \bar{c}_{l+1,t}(\boldsymbol{i})\leq 1$.

The \emph{cumulative expected normalized cost} of a non-empty subset $\mathcal{K}_{l+1}(\boldsymbol{i})$ up to iteration $t$ is defined as
\begin{equation*}
\bar{C}_{l+1,t}(\boldsymbol{i}) \overset{\text{def}}{=} \sum_{\tau=1}^{t} \bar{c}_{l+1,\tau}(\boldsymbol{i}).
\end{equation*}

Our proposed Recursive Exponential Weighting online algorithm is shown in Algorithm \ref{alg:alg-rew}.
In iteration $t$,we recursively select the subsets in all layers.
Given that a non-empty subset $\mathcal{I}_{l}(\bm{s})$ in $l$ layer is chosen.
In layer $(l+1)$, we first get the cumulative expected normalized cost up to slot $(t-1)$ for each non-empty subset $\mathcal{I}_{l+1}(\bm{i}) \subset \mathcal{I}_{l}(\bm{s})$, i.e.,
$\bar{C}_{l+1,t-1}(\boldsymbol{i})$.
Then we choose the subset $\mathcal{I}_{l+1}(\bm{i})$ with probability proportional to $\exp\left(-\eta_t \bar{C}_{l+1,t-1}(\boldsymbol{i})\right)$.
Note that the denominator in \eqref{equ:alg-rew-pt} in Algorithm \ref{alg:alg-rew} is a normalizer such that $\bm{p}_t$ is
a probability density function. After selecting the subsets in all layers, we further update expected normalized cost for all subsets in all layers at iteration $t$.

\begin{algorithm}[htp]
\caption{Recursive Exponential Weighting (\rew) Online Algorithm}
\label{alg:alg-rew}
\begin{algorithmic}[1]
\Require
index set $\mathcal{I}$, $T$, $\{\eta_{t}=\frac{1}{\sqrt{t}}\}$
\Ensure
\State Set $\bar{C}_{l,0}(\boldsymbol{i}) = 0$, for any $l=1,2,\ldots,m$
\For{$t=1$ to $T$}
    \State $\boldsymbol{s}=\boldsymbol{1}$
    \State {\texttt{//Recursively select the subsets in all layers}}
    \For{$l=0$ to $m-1$}
        \State Select a non-empty subset $\mathcal{I}_{l+1}(\boldsymbol{i}) \subset \mathcal{I}_{l}(\boldsymbol{s})$ with probability
                \be
                p_{t}(\mathcal{I}_{l+1}(\boldsymbol{i})) = \frac{\exp\left(-\eta_{t} \bar{C}_{l+1,t-1}(\boldsymbol{i}) \right)}
                {\sum_{\boldsymbol{i}\in \mathcal{D}_{l}(\boldsymbol{s})}  \exp\left(-\eta_{t} \bar{C}_{l+1,t-1}(\boldsymbol{i}) \right)}
                \label{equ:alg-rew-pt}
                \ee
        \If{subset $\mathcal{I}_{l+1}(\boldsymbol{i})$ is selected}
           \State   Set $\boldsymbol{s} = \boldsymbol{i}$
        \EndIf
    \EndFor
    \State {\texttt{//Get the expected normalized cost for all subsets in all layers}}
    \For{$l=0$ to $m-1$}
        \For {each non-empty subset $\mathcal{I}_{l+1}(\boldsymbol{i})$ in layer-$(l+1)$}
            \State Calculate
            $\Pr\left[\boldsymbol{I}_{t}=\boldsymbol{k}|\boldsymbol{I}_{t}\in \mathcal{I}_{l+1}(\boldsymbol{i})\right]$ based on Equation (\ref{eq:pro_arm})
            \State Calculate
            \begin{equation*}
            \begin{split}
            &\bar{c}_{l+1,t}(\boldsymbol{i}) \\
            =& \sum_{\boldsymbol{k}\in\mathcal{I}_{l+1}(\boldsymbol{i})}\!\!\frac{c_{t}(\boldsymbol{k})\!-\!\min_{\boldsymbol{k}\in \mathcal{U}_{l}(\boldsymbol{i})}c_{t}(\boldsymbol{k})}{n2^{m-l}L_{\mathsf{d}}}\cdot\Pr\left[\boldsymbol{I}_{t}\!=\!\boldsymbol{k}|\boldsymbol{I}_{t}\!\in \! \mathcal{I}_{l+1}(\boldsymbol{i})\right]
            \end{split}
            \end{equation*}
        \EndFor
    \EndFor
\EndFor
\end{algorithmic}
\end{algorithm}

\subsection{Regret Analysis for \rew} \label{subsec:regret-rew}

To ease the analysis, we ``split'' the regret of the online algorithm into two parts.
The first part is the regret due to ``imperfect choosing'' over $\mathcal{I}$, \textit{i.e.},
\begin{equation*}
\textsf{regret}_{\textsf{ImC}}\overset{\text{def}}{=}\sup_{f_{1},\ldots,f_{T}\in \mathcal{F}}\left\{\sum_{t\in\mathcal{T}} c_t(\bm{I}_{t})-\min_{\boldsymbol{i}\in \mathcal{I}}\sum_{t\in\mathcal{T}} c_t(\bm{i})\right\},
\end{equation*}
where the first term is the cumulative cost incurred by the online algorithm (whose choice at time slot $t$ is denoted by $\bm{I}_{t}$), and the second term is the minimum cumulative cost among representative points.
The second part of the regret is from ``imperfect discretization'', which is represented as
\begin{equation*}
\textsf{regret}_{\textsf{ImD}}\overset{\text{def}}{=}\sup_{f_{1},\ldots,f_{T}\in \mathcal{F}}\left\{\min_{\boldsymbol{i}\in \mathcal{I}}\sum_{t\in\mathcal{T}} c_t(\bm{i}) - \min_{\boldsymbol{x}\in \mathcal{K}}\sum_{t\in\mathcal{T}} f_t(\bm{x})\right\},
\end{equation*}
where the second term is the minimum cumulative cost over decision set $\mathcal{K}$. Obviously, we have that
\begin{equation*}
\textsf{regret}\xspace_{T}\leq \textsf{regret}_{\textsf{ImC}}+\textsf{regret}_{\textsf{ImD}}.
\end{equation*}

We now show the regret of the \rew algorithm for the subproblem of choosing point over the index set $\mathcal{I}$.
\begin{lemma} \label{lem:regret-imc}
The \rew algorithm guarantees that
\begin{equation*}
\textsf{\emph{regret}}_{\textsf{\emph{ImC}}}\leq \left(4n^2+\frac{1}{2}n\right) 2^{m}L_{\mathsf{d}} \sqrt{T}+ 2n^{2}\cdot 2^{m}L_{\mathsf{d}}.
\end{equation*}
\end{lemma}
\begin{proof}
It is easy to see that the \rew online algorithm has a layered structure to determine the final decision. Suppose a non-empty subset $\mathcal{I}_{l}(\boldsymbol{i})$ is chosen at layer $l\in\{0,1,2,\ldots,m-1\}$. In the next step, \rew will further choose a subset whose index lies in $\mathcal{D}_{l}(\boldsymbol{i})$. Among the subsets of $\mathcal{I}_{l}(\boldsymbol{i})$, there exists a local optimal subset in hindsight and potentially a regret loss due to imperfect choosing at layer $l$ will be incurred by the online algorithm.
Equation (\ref{eq:subset-regret}) expresses the regret loss at the $l$-th layer.
\begin{figure*}
\begin{equation}
\label{eq:subset-regret}
\begin{split}
&\sum_{t\in \mathcal{T}}\mathbb{E}[c_{t}(\boldsymbol{I}_{t})|\boldsymbol{I}_{t}\in \mathcal{I}_{l}(\boldsymbol{i})] \\
=&\sum_{t\in \mathcal{T}}\sum_{\boldsymbol{j}\in \mathcal{D}_{l}(\boldsymbol{i})} \Pr[\boldsymbol{I}_{t}\in \mathcal{I}_{l+1}(\boldsymbol{j})|\boldsymbol{I}_{t}\in \mathcal{I}_{l}(\boldsymbol{i})] \mathbb{E}\left[c_{t}(\boldsymbol{I}_{t})|\boldsymbol{I}_{t}\in \mathcal{I}_{l+1}(\boldsymbol{j})\right] \\
\overset{(a)}{=}&\sum_{t\in \mathcal{T}}\sum_{\boldsymbol{j}\in \mathcal{D}_{l}(\boldsymbol{i})} \frac{\exp\left(-\eta_{t} \bar{C}_{l+1,t-1}(\boldsymbol{j})\right)}{\sum_{\boldsymbol{j}\in \mathcal{D}_{l}(\boldsymbol{i})}\exp\left(-\eta_{t} \bar{C}_{l+1,t-1}(\boldsymbol{j})\right)}\cdot \left[\bar{c}_{l+1,t}(\boldsymbol{j})\cdot n2^{m-l}L^{\mathsf{d}}+\min_{\boldsymbol{k}\in \mathcal{I}_{l}(\boldsymbol{i})}c_{t}(\boldsymbol{k})\right] \\
\overset{(b)}{\leq} &\sum_{t\in \mathcal{T}}\left[-\frac{1}{\eta_{t}}\ln \sum_{\boldsymbol{j}\in \mathcal{D}_{l}(\boldsymbol{i})}\frac{\exp\left(-\eta_{t}\bar{C}_{l+1,t-1}(\boldsymbol{j})\right)}{\sum_{\boldsymbol{j}\in \mathcal{D}_{l}(\boldsymbol{i})}\exp\left(-\eta_{t} \bar{C}_{l+1,t-1}(\boldsymbol{j})\right)}\exp\left(-\eta_{t} \bar{c}_{l+1,t}(\boldsymbol{j})\right)+\frac{\eta_{t}}{8}\cdot 1^{2}\right]\cdot n2^{m-l}L_{\mathsf{d}} \\
&~~~~~~~~~~~~~~~~~~~~~~~~~~~~~~~~~~~~~~~~~~~~~~~~~~~~~~~~~~~~~~~~~~~~~~~~~~~~~~~~~~~~~~~~~~~~~~~~~~~~~~~~+ \sum_{t\in \mathcal{T}}\min_{\boldsymbol{k}\in \mathcal{I}_{l}(\boldsymbol{i})}c_{t}(\boldsymbol{k})\\
=&\sum_{t\in \mathcal{T}}\left[-\frac{1}{\eta_{t}}\ln \sum_{\boldsymbol{j}\in \mathcal{D}_{l}(\boldsymbol{i})}\frac{\exp\left(-\eta_{t}\bar{C}_{l+1,t}(\boldsymbol{j})\right)}{\sum_{\boldsymbol{j}\in \mathcal{D}_{l}(\boldsymbol{i})}\exp\left(-\eta_{t} \bar{C}_{l+1,t-1}(\boldsymbol{j})\right)}+\frac{\eta_{t}}{8}\cdot 1^{2}\right]\cdot n2^{m-l}L_{\mathsf{d}} + \sum_{t\in \mathcal{T}}\min_{\boldsymbol{k}\in \mathcal{I}_{l}(\boldsymbol{i})}c_{t}(\boldsymbol{k})\\
\overset{(c)}{=} &\sum_{t\in \mathcal{T}}\left[\Phi_{t}(\eta_{t})-\Phi_{t-1}(\eta_{t}) +\frac{\eta_{t}}{8}\cdot 1^{2}\right]\cdot n2^{m-l}L_{\mathsf{d}} + \sum_{t\in \mathcal{T}}\min_{\boldsymbol{k}\in \mathcal{I}_{l}(\boldsymbol{i})}c_{t}(\boldsymbol{k})\\
\overset{(d)}{\leq} & \left\{\Phi_{T}(\eta_{T})+\sum_{t=2}^{T}\left[ n\left(\frac{1}{\eta_{t}}-\frac{1}{\eta_{t-1}}\right)+\frac{\eta_{t}}{8}\right]+\frac{1}{\eta_{1}}\ln |\mathcal{D}_{l}(\boldsymbol{i})|\right\}\cdot n2^{m-l}L_{\mathsf{d}} + \sum_{t\in \mathcal{T}}\min_{\boldsymbol{k}\in \mathcal{I}_{l}(\boldsymbol{i})}c_{t}(\boldsymbol{k})\\
\overset{(e)}{\leq} & \left\{\sum_{t\in \mathcal{T}}\mathbb{E}\left[\frac{c_{t}(\boldsymbol{I}_{t})-\min_{\boldsymbol{k}\in \mathcal{I}_{l}(\boldsymbol{i})}c_{t}(\boldsymbol{k})}{n2^{m-l}L_{\mathsf{d}}}|\boldsymbol{I}_{t}\in \mathcal{I}_{l+1}(\boldsymbol{j})\right] \right.\\
&~~~~~~~~~~~~~~~~~~~~~~~~~\left.+\frac{1}{\eta_{T}}n+\sum_{t=2}^{T}\left[ n\left(\frac{1}{\eta_{t}}-\frac{1}{\eta_{t-1}}\right)+\frac{\eta_{t}}{8}\right]+\frac{1}{\eta_{1}}\ln |\mathcal{D}_{l}(\boldsymbol{i})|\right\}\cdot n2^{m-l}L_{\mathsf{d}}+ \sum_{t\in \mathcal{T}}\min_{\boldsymbol{k}\in \mathcal{I}_{l}(\boldsymbol{i})}c_{t}(\boldsymbol{k})\\
\overset{(f)}{\leq} &\sum_{t\in \mathcal{T}}\mathbb{E}\left[c_{t}(\boldsymbol{I}_{t})|\boldsymbol{I}_{t}\in \mathcal{I}_{l+1}(\boldsymbol{j})\right]+\left(2n+\frac{1}{4}\right)\sqrt{T}\cdot n2^{m-l}L_{\mathsf{d}} + \frac{1}{\eta_{1}}\ln |\mathcal{D}_{l}(\boldsymbol{i})|\cdot n2^{m-l}L_{\mathsf{d}}
\end{split}
\end{equation}
\end{figure*}

In Equation (\ref{eq:subset-regret}), Equality (a) is based on the definition for $\bar{c}_{l+1,t}(\boldsymbol{j})$.
Inequality (b) is by Hoeffding's Lemma and the fact that $\bar{c}_{l+1,t}(\boldsymbol{j})\in [0,1]$.
Equality (c) simplifies the expression for the log-sum-exp function by defining
\begin{equation*}
\Phi_{t}(\eta_{t})=-\frac{1}{\eta_{t}}\ln \sum_{\boldsymbol{j}\in \mathcal{D}_{l}(\boldsymbol{i})}\exp\left(-\eta_{t}\bar{C}_{l+1,t}(\boldsymbol{j})\right).
\end{equation*}
Inequality (d) rearranges the first set of terms and uses the following bound results
\begin{equation*}
\sum_{t=1}^{T-1}[\Phi_{t}(\eta_{t})-\Phi_{t}(\eta_{t+1})] \leq \sum_{t=1}^{T-1}n\left(\frac{1}{\eta_{t+1}}-\frac{1}{\eta_{t}}\right)\leq n\sqrt{T}.
\end{equation*}
Inequality (e) uses the approximation results of a log-sum-exp function for the minimum value in a discrete set (see \cite{Chen2013}).
In Inequality (f), we use the result that
\begin{equation*}
\begin{split}
\sum_{t=1}^{T} \frac{\eta_{t}}{8}  < & \sum_{t=1}^{T} \frac{1}{4\left(\sqrt{t}+\sqrt{t+1}\right)}  \\
= & \sum_{t=1}^{T} \frac{1}{4}\left(\sqrt{t+1}-\sqrt{t}\right) \\
=& \frac{1}{4}\left(\sqrt{T+1}-1\right)\leq \frac{1}{4} \sqrt{T}. \nnb
\end{split}
\end{equation*}

Let $\boldsymbol{i}^{*}$ be the index of the arm of the smallest cumulative cost in the index set $\mathcal{I}$, \textit{i.e.},
\begin{equation*}
\boldsymbol{i}^{*} \overset{\text{def}}{=} {\arg\min}_{\boldsymbol{i}\in \mathcal{I}} \sum_{t\in \mathcal{T}}c_{t}(\boldsymbol{i}).
\end{equation*}
The following equation upper bounds $\textsf{regret}_{\textsf{ImC}}$.
\begin{equation*}
\begin{split}
&\textsf{regret}_{\textsf{ImC}} \\
=& \sum_{t\in \mathcal{T}}\mathbb{E}[c_{t}(\boldsymbol{I}_{t})]-\sum_{t\in \mathcal{T}}c_{t}(\boldsymbol{i}^{*}) \\
=& \sum_{t\in \mathcal{T}}\sum_{\boldsymbol{j}\in \mathcal{D}_{1}} \Pr[\boldsymbol{I}_{t}\in \mathcal{I}_{1}(\boldsymbol{j})] \mathbb{E}[c_{t}(\boldsymbol{I}_{t})|\boldsymbol{I}_{t}\in \mathcal{I}_{1}(\boldsymbol{j})] -\sum_{t\in \mathcal{T}}c_{t}(\boldsymbol{i}^{*})\\
\leq & \sum_{t\in \mathcal{T}}\mathbb{E}\left[c_{t}(\boldsymbol{I}_{t})|\boldsymbol{I}_{t}\in \mathcal{M}_{1}(\boldsymbol{i}^{*})\right] -\sum_{t\in \mathcal{T}}c_{t}(\boldsymbol{i}^{*}) +\left(2n+\frac{1}{4}\right)\sqrt{T}\cdot n2^{m}L_{\mathsf{d}} + \frac{1}{\eta_{1}}\ln |\mathcal{D}_{l}(\boldsymbol{i})|\cdot n2^{m}L_{\mathsf{d}} \\
\leq &n\left(2n+\frac{1}{4}\right)\sqrt{T}\cdot 2^{m}L_{\mathsf{d}} +n\left(2n+\frac{1}{4}\right)\sqrt{T}\cdot 2^{m-1}L_{\mathsf{d}}+\cdots+n\left(2n+\frac{1}{4}\right)\sqrt{T}\cdot L_{\mathsf{d}} + 2n^{2}\cdot 2^{m}L_{\mathsf{d}}\\
\leq &\left(4n^2+\frac{1}{2}n\right)\cdot 2^{m}L_{\mathsf{d}} \sqrt{T}+ 2n^{2}\cdot 2^{m}L_{\mathsf{d}}.
\end{split}
\end{equation*}

This completes the proof.
\end{proof}

Note that $L_{\mathsf{d}}=2\sqrt{n}DL/(2^m)$. Combined with Lemma \ref{lem:regret-imc}, we have
\begin{equation*}
\textsf{regret}_{\textsf{ImC}}\leq \left(8n^2\sqrt{n}+n\sqrt{n}\right)\cdot DL \sqrt{T}+ 4n^{2}\sqrt{n}\cdot DL.
\end{equation*}
This implies that $\textsf{regret}_{\textsf{ImC}}$ is always upper bounded by $O(\sqrt{T})$, no matter what value $m$ takes. On the other hand, the regret loss due to imperfect discretization can be reduced with larger $m$. Specifically,
\begin{equation*}
\textsf{regret}_{\textsf{ImD}}\leq \frac{\sqrt{n}}{2^{m}}LDT.
\end{equation*}

Based on the above results, we conclude our main results in Theorem \ref{the:regret-non-convex-problem}.
\begin{theorem} \label{the:regret-non-convex-problem}
The \rew algorithm guarantees that
\begin{equation*}
\textsf{\emph{regret}}_{T}\leq \left(8n^2\sqrt{n}+n\sqrt{n} \right)\cdot DL \sqrt{T}+\frac{\sqrt{n}}{2^{m}}DLT+ 4n^{2}\sqrt{n}\cdot DL.
\end{equation*}
With $m$ being set to $\ln \sqrt{T}$, we have
\begin{equation*}
\textsf{\emph{regret}}_{T}\leq \left(8n^2\sqrt{n}+n\sqrt{n} + \sqrt{n}\right)\cdot DL \sqrt{T}+ 4n^{2}\sqrt{n}\cdot DL.
\end{equation*}
\end{theorem}

Theorem \ref{the:regret-non-convex-problem} implies that the \rew algorithm attains a regret of $O(n^2\sqrt{T})$, which is well known to be the lower bound of the regret \cite{Hazan16}.
\section{Numerical Results}
We proceed to test our algorithm on a numerical example in $\mathbb{R}^1$ with a class of piecewise cost functions defined in region $[-1,1]$. The form of the cost function is depicted in Figure \ref{fig:loss-function}.
\begin{figure}[h]
\center
  \begin{minipage}[t]{0.6\linewidth}
    \centering
    \includegraphics[width = \linewidth]{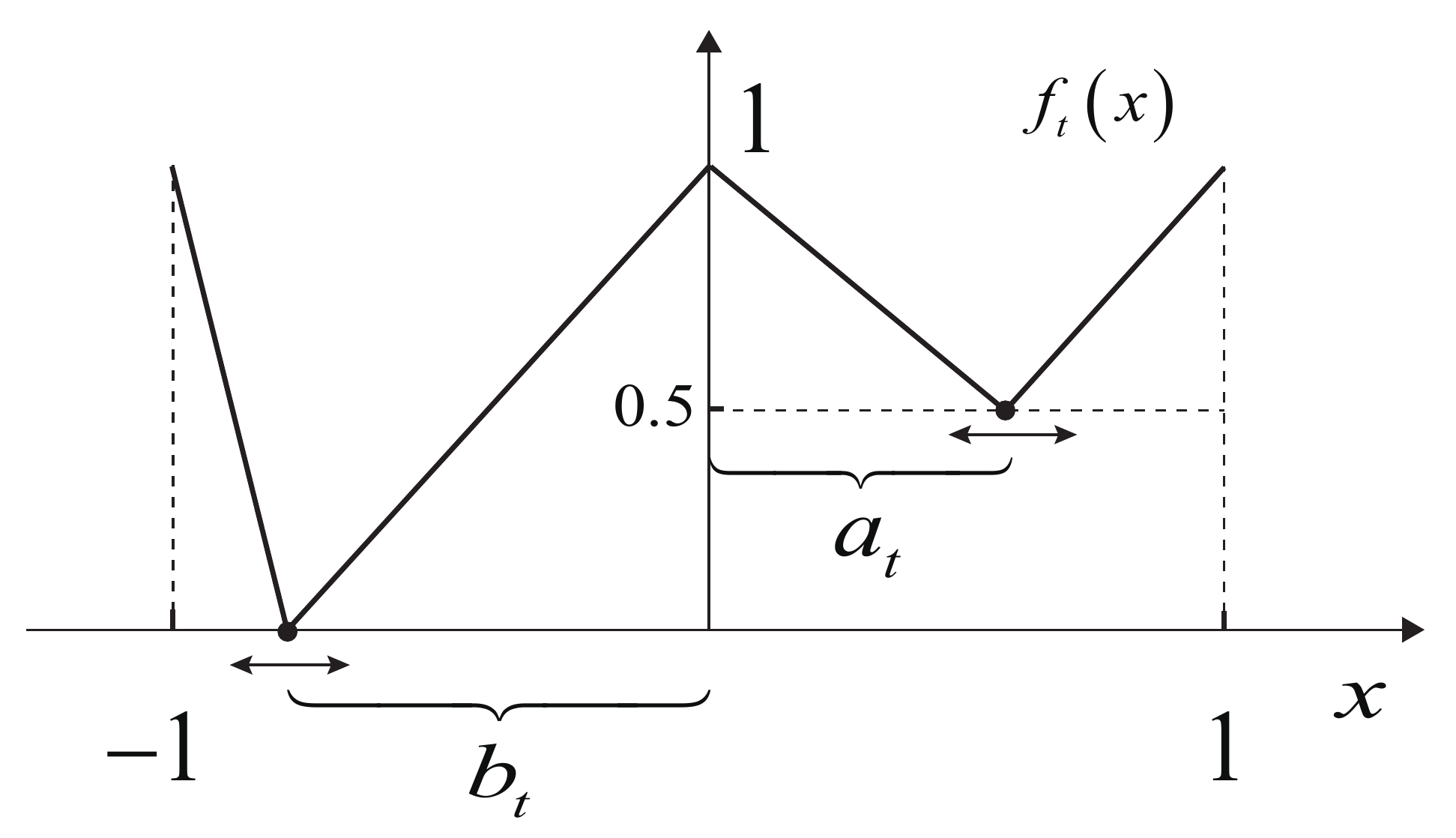}\\
  \caption{The loss function.}
  \label{fig:loss-function}
  \end{minipage}
\end{figure}

The loss function can be also expressed as the following.
\begin{equation*}
f_{t}(x)=\left\{
\begin{array}{ll}
1-\frac{1}{1-b_t}(x+1)& x\in[-1,-b_t), \\
\frac{1}{b_t}(x+b_t) & x\in [-b_t,0), \\
1-\frac{1}{a_t}(x) & x\in [0,a_t), \\
\frac{1}{1-a_t}(x-a_t) & x\in[a_t,1].
\end{array}\right.
\end{equation*}

As depicted in Figure \ref{fig:loss-function}, the minimum value of $f_t(x)$ over region $[-1,0]$ and $[0,1]$ is $0$ and $0.5$, respectively. $a_t$ and $b_t$ denote the positions of the two minimum values over region $[0,1]$ and $[-1,0]$.
In each time slot, the adversary can choose any value from $[\frac{1}{3},\frac{2}{3}]$ for $a_t$ and $b_t$.
$a_t$ or $b_t$ is unknown at the beginning of a time slot $t$ and will be revealed to the player after the commitment of a choice. As an example, we assume the adversary randomly chooses $a_t$ and $b_t$ at uniform property over $[\frac{1}{3},\frac{2}{3}]$.

For a simulation horizon of $T=3600$, we set the discretization parameter $m$ to be $8$, which is lager than $\ln \sqrt{T}$. We compare our online algorithm with the online Gradient Decent (OGD) method with initial point being $1$.
Figure \ref{fig:performance} shows the convergence performance of the \rew algorithm and the comparison algorithm. We report the ``empirical time-average regret'' as the performance metric, which is obtained by dividing the cumulative cost difference of each algorithm by the duration time. \rev{We only show the comparison results at the first 300 time slots.}

\begin{figure}[h]
\center
  \begin{minipage}[t]{0.7\linewidth}
    \centering
    \includegraphics[width = \linewidth]{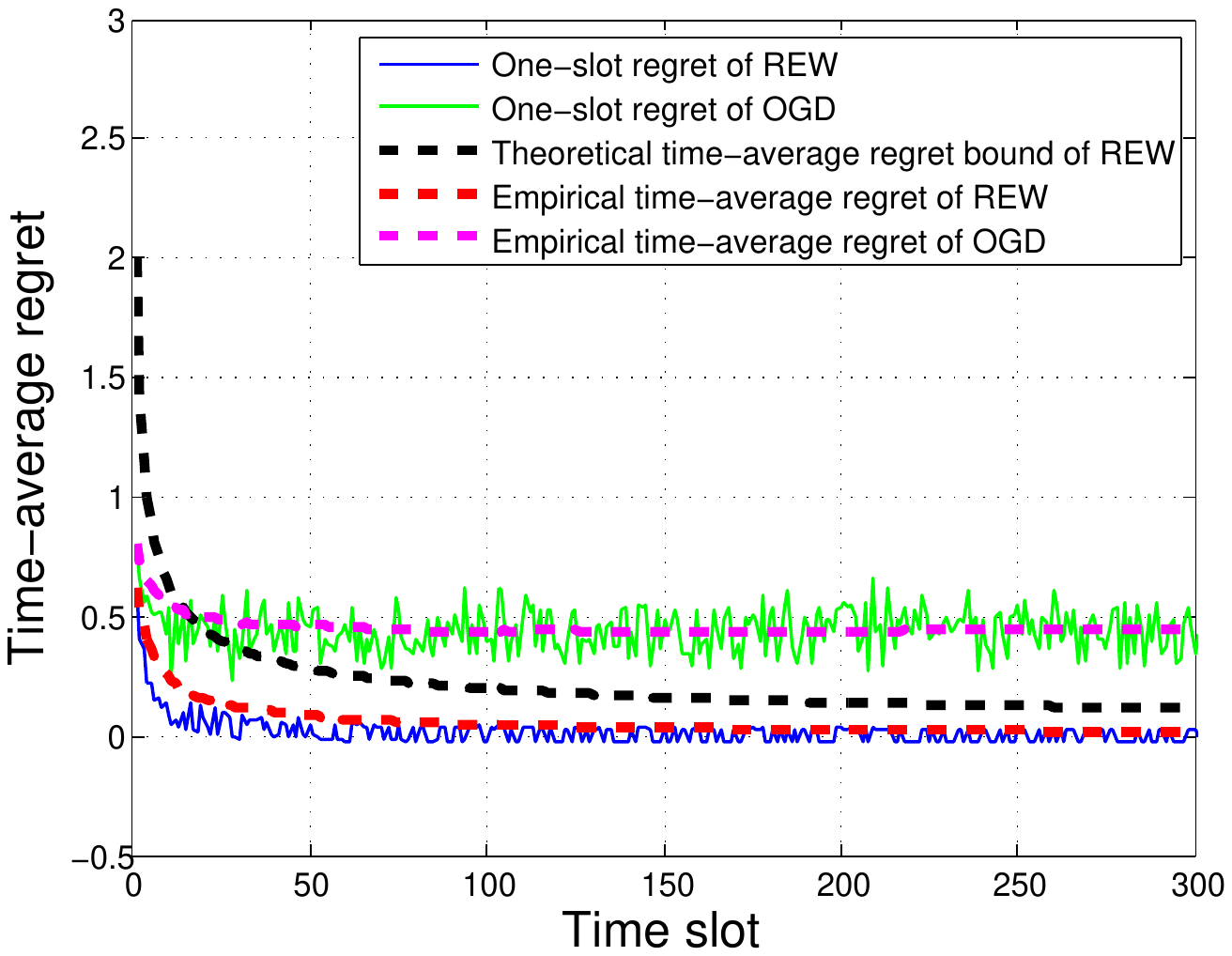}\\
  \caption{Time-average empirical regret and the theoretical bound.}
  \label{fig:performance}
  \end{minipage}
\end{figure}

There are two important observations in Figure \ref{fig:performance}. The first one is that the empirical convergence rate is far better than the regret bound. That is because the regret bound is achieved only under the worst-case inputs which are far more sophisticated than the ``average-case'' ones. The other observation is that the traditional online convex optimization method, like the Online Gradient Decent method, may lead to a sub-optimal solution. In our test, the cost function in each time slot is obtained by concatenating two convex functions in region $[-1,0]$ and $[0,1]$. We deliberately set the initial point to be $1$, and ultimately the the OGD algorithm converges its solution to a local optimum within the right half of the decision region.

\section{Discussions on Possible Extensions}

Recently, Hosseini \cite{Hosseini2016} and Lee \cite{Lee2017} generalized the classic OCO problem to a decentralized optimization framework within a network of agents. An interesting result of their work is that the $O(\sqrt{T})$ regret can be still attained by leveraging the communication among agents.
In addition to that, \cite{shahrampour2017distributed} addresses decentralized online optimization
in non-stationary environments using mirror decent, and in \cite{akbari2015distributed}, distributed online optimization is studied for strongly convex objective functions over time-varying networks.

Parallelly, a promising future work is to implement the proposed weighting method to solve the distributed online non-convex optimization problem. The extension is natural, while possible techniques to be adopted might be rather different. That is because the \rev{\rew} algorithm maintains the estimation for each subset and it might be very costly to exchange such information. Thus, in the opinion of the authors, the main issue of implementing such a weighting method within a decentralized environment is to alleviate the communication overhead among agents.

We should emphasize that our online non-convex problem is based on \emph{full information feedback}.
Namely, the whole cost function will be revealed at the ending time of each time slot. It is an interesting and important future direction to consider {\emph{partial information feedback} where only the cost value of the player's choice is revealed.
The {\emph{partial information feedback} extension is motivated by many real-world systems in which the observer is not co-located with the controller and the feedback information is noisy, partial or incomplete due to limited communication bandwidth.
Such are the cases in online routing in data networks \cite{Awerbuch08}, power control in cellular networks \cite{Wong04} and the ad placement problem on a web page \cite{Wauthier13}.
%

\cite{Flaxman05,Abernethy08,Saha11,Dekel15} studied the bandit information feedback setting in the OCO model with a sub-linear regret obtained, respectively. However, none of them have attained the lower bound of the regret. For the special case of strongly-convex and smooth losses, \cite{Agarwal10} obtained a regret of $\tilde{O}(\sqrt{T})$ in the unconstrained case, and \cite{Hazan14} obtained the same rate even in the constrained case. \cite{Shamir13} gave a lower bound of $O(\sqrt{T})$ for the setting of strongly-convex and smooth BCO. Recently, a new algorithm was reported by Hazan \textit{et al.} to attain a regret of $(\ln T)^{2d}\sqrt{T}$ \cite{Hazan16-1}. This is the first algorithm to attain a $\tilde{O}(\sqrt{T})$ regret for the OCO model with bandit feedback. \cite{kleinberg2013bandits} studied the model where the action set is from a metric space, and the payoff function satisfies a Lipschitz condition with
respect to the metric. Their results show that there is an algorithm whose regret on any instance satisfies $R(T)=\tilde{O}(T^{\frac{d+1}{d+2}})$ for every $T$,
where $d$ is the dimension of the action set.

Despite of the above results, the optimal online algorithm and tight regret bound for the online convex/non-convex optimization problem with bandit feedback are still open, calling for more investigation from the community.


\section{Conclusion}
In this paper, we investigated the online nonconvex optimization problem, which removes the convexity assumption of the cost functions as compared with the online convex optimization problem.
This generalization makes it far more challenging to design an efficient online algorithm of sublinear regret. Our results shows that by properly partitioning subsets and using the recursive exponential weighting method, the regret can be reduced to match the lower bound, $\sqrt{T}$.


\bibliographystyle{abbrv}

\end{document}